\documentclass[preprint]{elsarticle}
\usepackage{times}
\usepackage{amsthm,amssymb,amsmath}
\usepackage{amsfonts}
\usepackage{longtable}
\usepackage{footnote}
\usepackage{url}
\usepackage{subfig}
\usepackage{fmtcount}
\usepackage{enumerate}
\usepackage{float}
\usepackage{setspace}
\usepackage{lineno}
\floatstyle{plaintop}
\restylefloat{table}

\newcommand*\patchAmsMathEnvironmentForLineno[1]{%
  \expandafter\let\csname old#1\expandafter\endcsname\csname #1\endcsname
  \expandafter\let\csname oldend#1\expandafter\endcsname\csname end#1\endcsname
  \renewenvironment{#1}%
     {\linenomath\csname old#1\endcsname}%
     {\csname oldend#1\endcsname\endlinenomath}}%
\newcommand*\patchBothAmsMathEnvironmentsForLineno[1]{%
  \patchAmsMathEnvironmentForLineno{#1}%
  \patchAmsMathEnvironmentForLineno{#1*}}%
\AtBeginDocument{%
\patchBothAmsMathEnvironmentsForLineno{equation}%
\patchBothAmsMathEnvironmentsForLineno{align}%
\patchBothAmsMathEnvironmentsForLineno{flalign}%
\patchBothAmsMathEnvironmentsForLineno{alignat}%
\patchBothAmsMathEnvironmentsForLineno{gather}%
\patchBothAmsMathEnvironmentsForLineno{multline}%
}

\def\b0{{\boldsymbol{0}}}

\newtheorem{thm}{Theorem}

\newtheorem{lem}{Lemma}

\journal{Pattern Recognition Letters}

\begin{document}

\begin{frontmatter}
\title{On the Convergence of the Mean Shift Algorithm in the One-Dimensional Space}

\author{Youness Aliyari Ghassabeh\corref{cor1}}
\ead{aliyari@cs.toronto.edu}

\cortext[cor1]{Corresponding author: Phone +13433334863;
}

\address{Department of Mathematics and Statistics, Queen's University, Kingston, ON, K7L 3N6}

\begin{abstract}
 The mean shift algorithm is a non-parametric and iterative technique that has been used for finding modes of an estimated probability density function. It has been successfully employed in many applications in specific areas of machine vision, pattern recognition, and image processing. Although the mean shift algorithm has been used in many applications, a rigorous proof of its convergence is still missing in the literature. In this paper we address the convergence of the mean shift algorithm in the one-dimensional space and prove that the sequence generated by the mean shift algorithm is a monotone and convergent sequence.
\end{abstract}

\end{frontmatter}

\section{Introduction}
The mean shift algorithm is a simple, non-parametric, iterative method introduced by Fukunaga and Hostetler \cite{Fukunaga} for finding modes of an estimated probability density function (pdf).
Modes of an estimated pdf play an important role in many pattern recognition applications such as image segmentation \cite{segment}, classification\cite{Youness1}, feature extraction \cite{Youness}, and object tracking \cite{tracking}. The mean shift algorithm was generalized by Cheng \cite{Cheng} and became popular in the machine vision community when its potential uses for feature space analysis were studied \cite{meer}.
In recent years, the mean shift algorithm has been successfully used in many applications ranging from image segmentation \cite{imageseg} to object tracking \cite{application2}\cite{application3}, edge detection \cite{application5}\cite{application66}, information fusion \cite{application4}, and vector quantization \cite{ourpaper2}.

The mean shift algorithm shifts each data point to the weighted average of the data set and tries to find stationary points of an estimated pdf. It starts from one of the data points and iteratively improves the mode estimate. In contrast to the $k$-mean clustering approach \cite{bookpattern} (or other classification techniques \cite{conf1}\cite{conf2}), the mean shift algorithm does not require any prior knowledge of the number of clusters and there is no assumption of the shape of the clusters.
It has been claimed that the mean shift procedure generates a convergent sequence \cite{meer}. But a crucial step in the proof given in \cite{meer} for the convergence of the mode estimate sequence is not correct. Authors in \cite{meer} claimed that the generated sequence is a Cauchy sequence, which is not true in general.
In another work \cite{per}, it was shown that the mean shift algorithm is an expectation maximization (EM) algorithm and hence the generated sequence converges to the modes of the estimated pdf. However, the EM algorithm may not converge (e.g., a counterexample is given in \cite{emnot}), in which case, the convergence of the mean shift algorithm does not follow.
On the positive side, the authors in \cite{finite}\cite{ML} assumed that the number of the stationary points of the estimated pdf inside the convex hull of the data set is finite (or equivalently the stationary points inside the convex hull of the data set are isolated) and using this assumption they proved the convergence of the iterative algorithm\cite{finite}. However, the authors in \cite{finite}\cite{ML} could not justify their assumption and the finiteness of the stationary points of the estimated pdf with the widely used kernels (e.g. Gaussian kernel) has not been shown. For the one-dimensional case, Shieh \emph{et al.} \cite{one_dim} tried to find sufficient conditions to avoid premature convergence of the sequence, but they did not show that the sequence converges to a mode of the estimated pdf.

The authors in \cite{PR} showed the convergence of the mode estimate sequence in the one-dimensional space when the MS algorithm uses the Gaussian kernel. In this paper, we investigate the convergence of the mean shift algorithm for a wide class of kernels (not necessarily the Gaussian kernel) and prove the convergence of the mean shift algorithm in the one dimension with a convex and non-increasing kernel. In contrast to \cite{finite} we do not put any constraint on the number of the stationary points of the estimated pdf. The organization of the paper is as follows: in the next section we give a brief review of the mean shift algorithm. The convergence proof for the one-dimensional mean shift algorithm is given in Section~\ref{section3}. Section~\ref{section4} is devoted to the simulations to confirm the theoretical results given in Section \ref{section3}. Finally, the concluding remarks are given in Section~\ref{section5}.

\section{Mean shift algorithm}\label{section2}
Let $x_{i}\in \mathbb{R}, i=1,\ldots,n$\footnote{Since the main contribution of this paper is showing the convergence of the mean shift algorithm in the one-dimensional case, we assumed that the data points are sampled uniformly from an unknown one-dimensional pdf.} be a sequence of $n$ independent and identically distributed (iid) random variables.
 Let $K$ denote a radially symmetric kernel function defined by $K(x)=c_{k,D}k(x^2)$, where $c_{k,D}$ is a normalization factor and $k:[0,\infty)\rightarrow [0,\infty)$ is the differentiable profile of the kernel. The kernel function $K$ is a non-negative, real valued and integrable function satisfying $\int_{\mathbb{R}}K(x)dx=1$. The profile of the kernel is assumed to be a non-negative, non-increasing and piecewise continuous function that satisfies $\int_{\mathbb{R}}k(x^2)dx<\infty$.
Two widely used profile functions are given by $k_{E}= 1-x \text{ if } 0\le x \le 1, \text{ otherwise } k_{E}= 0$ and $k_{N}(x)=\exp(-1/2x)$. The estimated pdf using the profile $k$ and the bandwidth $h$ is given by \cite{silverman}
\begin{equation}\label{one}
\hat{f}_{h,k}(x)=\frac{c_{k,D}}{nh}\sum_{i=1}^{n}k\big((\frac{x-x_{i}}{h})^2\big).
\end{equation}
\noindent Taking the derivative of (\ref{one}) and equating it to zero reveals that if $x^*$ is a mode of the estimated pdf, then it satisfies the following equality
\begin{equation}
x^*=\frac{\sum_{i=1}^nx_{i}g\big((\frac{x^*-x_{i}}{h})^2\big)}
{\sum_{i=1}^ng\big((\frac{x^*-x_{i}}{h})^2\big)},
\end{equation}
where $g(x)=-k'(x)$. Hence, the modes of the estimated pdf are fixed points of ($2$). The mean shift at point $x$ is defined by
\begin{equation}\label{fixed}
m_{h,g}(x)=\frac{\sum_{i=1}^nx_{i}g\big((\frac{x-x_{i}}{h})^2\big)}
{\sum_{i=1}^ng\big((\frac{x-x_{i}}{h})^2\big)}-x,
\end{equation}
where the scalar $m_{h,g}(x)$ is called mean shift scalar \cite{meer}. The mean shift algorithm  generates the mode estimate sequence $\{y_{j}\}_{j=1,2,\ldots}$ in order to estimate $x^*$ in ($2$), where $x^*$ is a mode of the estimated pdf. 
The mean shift algorithm starts from one of the data points ($y_{1}$ is initialized to one of the data points) and iteratively update this point to find modes of the estimated pdf. The mode estimate in ($j+1$)th iteration is updated by $y_{j+1}=y_{j}+m_{h,g}(y_{j})$, where $m_{h,g}(y_{j})$ is computed using ($3$). The mode update in ($j+1$)th iteration can be simplified to
\begin{align}\label{modeestimate}
y_{j+1}&=y_{j}+m_{h,g}(y_{j})\nonumber\\
&=y_{j}+\frac{\sum_{i=1}^nx_{i}g\big((\frac{y_{j}-x_{i}}{h})^2\big)}
{\sum_{i=1}^ng\big((\frac{y_{j}-x_{i}}{h})^2\big)}-y_{j}\nonumber\\
&=\frac{\sum_{i=1}^nx_{i}g((\frac{y_{i}-x_{i}}{h})^2)}{\sum_{i=1}^ng((\frac{y_{i}-x_{i}}{h})^2)}.
\end{align}
The estimated mode update is iterated until the Euclidean distance between two consecutive mode estimates becomes less than some predefined epsilon.

\noindent The mean shift algorithm is an instance of the gradient ascent algorithm with an adaptive step size \cite{bound}, and in each iteration it tries to improve the previous estimation. The algorithm is applied to all data points, and it is expected to converge to the stationary points of the estimated pdf. Finally, the stationary points are pruned by retaining only the local maxima of the estimated pdf \cite{meer}.  It can be shown that the Euclidean distance between two consecutive mode estimates converges to zero as the number of iteration goes to infinity, i.e., $\lim_{j\rightarrow \infty} (y_{j+1}-y_{j})^2=0$ \cite{meer}\cite{finite}. However, this property does not imply the convergence of the mode estimate sequence $\{y_{j}\}$, and the convergence of the sequence needs to be proved separately. In the next section we prove the convergence of the mode estimate sequence $\{y_{j}\}$ in the one-dimensional case.

\section{Convergence of the mean shift algorithm in one dimension}\label{section3}

We prove the following theorem
\begin{thm}

Let $X=\{x_{1}, x_{2},\ldots, x_{n}\}$  denote the input data. Let $\hat{f}_{h,k}(x)$ denote the estimated pdf using a kernel $K$ with a convex, differentiable, and strictly decreasing profile $k$ and the bandwidth $h$. Suppose that $g(x)=-k'(x)$ is a strictly decreasing function, then the mode estimate sequence generated by the mean shift algorithm converges.
\end{thm}
\begin{proof}
Since the mode estimate sequence is bounded, it suffices to show that it is a monotone sequence. We prove if for all $x_{i}\in X$, $\hat{f}_{h,k}^{'}(x_{i})\neq 0$ and $\hat{f}_{h,k}^{'}(x)$ is a continuous function, then there exits $N$ such that for all $j>N$, the mode estimate sequence $\{y_{j}\}$ will be a monotone sequence. If for some $x_{i} \in X$, $\hat{f}_{h,k}^{'}(x_{i})=0$, then either the mode estimate sequence converges to those $x_{i}$'s or there exists a large enough $N$ such that for all $j>N$ the mode estimate sequence $\{y_{j}\}$ will be a monotone sequence.
The following inequality was proved in \cite{meer}

\begin{equation*}
\hat{f}_{h,k}(y_{j+1})-\hat{f}_{h,k}(y_{j})\geq \frac{c_{k}}{nh^{2}}\|y_{j+1}-y_{j}\|^2\sum_{i=1}^n
g\Big(\|\frac{y_{j}-x_{i}}{h}\|^2\Big),
\end{equation*}

\noindent where $g(x)=-k'(x)$ and $c_{k}$ is the normalization factor.
Let $M(j)=\min\{g(\|\frac{y_{j}-x_{i}}{h}\|^2) ,i=1, \dots, n\}$. We have $M(j)\geq g(\frac{d^2}{h^2})$, where $d$ denotes the supremum of the pairwise distances between elements of $X$, i.e., $d=\sup\{ |x_{i}-x_{j}|, i,j=1,\ldots,n, i\ne j\}$. Let $\varphi=g(\frac{d^2}{h^2})$.  Hence, the above inequality can be simplified as follows
\begin{align*}
\hat{f}_{h,k}(y_{j+1})-\hat{f}_{h,k}(y_{j})&\geq \frac{c_{k}}{nh^{2}}(y_{j+1}-y_{j})^2\sum_{i=1}^n
g\Big((\frac{y_{j}-x_{i}}{h})^2\Big) \nonumber\\
&\geq \frac{c_{k}}{nh^{2}}(y_{j+1}-y_{j})^2 n M(j)  \nonumber\\
&\geq\frac{c_{k}}{h^{2}}(y_{j+1}-y_{j})^2  \varphi.
\end{align*}
 Therefore, we have
\begin{align*}
\Big(\hat{f}_{h,k}(y_{j+1})-\hat{f}_{h,k}(y_{j}) \Big)\frac{h^{2}}{\varphi c_{k}} \geq (y_{j+1}-y_{j})^2 \geq 0.
\end{align*}
Since $\hat{f}_{h,k}(y_{j+1})$ is a convergent sequence \cite{meer}, the limit of the left side of the above inequality as $j \rightarrow \infty$ is zero. Therefore, the following limit relation holds
\begin{align}\label{consequtive difference}
\lim_{j\rightarrow \infty} |y_{j+1}-y_{j}| =0.
\end{align}
The following equality also was proved in \cite{meer}
\begin{align}\label{deri}
\lim_{j\rightarrow \infty}\hat{f}^{'}_{h,k}(y_{j})=0.
\end{align}
 Now we consider the case that $\hat{f}^{'}_{h,k}(x_{i})\neq 0, \forall x_{i} \in X$.
  For all $x_{i}\in X$, $\hat{f}^{'}_{h,k}(x_{i})\neq 0$, as a result of which there exists  $\epsilon_{i}>0$ such that $\hat{f}^{'}_{K}(x)$ is nonzero in the closed interval centered at $x_{i}$ with radius $\epsilon_{i}$, denoted by $I[x_{i},\epsilon_{i}], i=1,\ldots, n$. Let
  $\epsilon=\min\{ \epsilon_{i} ,\; i=1,\ldots,n\}$.
  Since $\hat{f}^{'}_{h,k}(x)$ is continuous, it achieves its minimum over the compact set  $\bigcup_{i=1}^n I[x_{i},\epsilon]$, so let $c=\min_{x\in\bigcup_{i=1}^n I[x_{i},\epsilon]}\hat{f}^{'}_{h,k}(x)$. By assumption, it is clear that $c>0$. From (\ref{consequtive difference}) the sequence $\{|y_{j+1}-y_{j}|\}_{j=1,2,\ldots}$ converges to zero. Therefore, for every $\epsilon/2>0$, there exists a constant $N_{1}(\epsilon/2)>0$  such that for all $j$ greater than $N_{1}(\epsilon/2)$, the difference between two consecutive mode estimates becomes less than $\epsilon/2$, i.e., $|y_{j+1}-y_{j}|<\epsilon/2,\forall j>N_{1}(\epsilon/2)$\footnote{The upper bound $N_{1}(\epsilon/2)$ for $|y_{j+1}-y_{j}|$ comes from the the convergence of the sequence $\{|y_{j+1}-y_{j}|\}_{j=1,2,\ldots}$ to zero. By definition, if a sequence $\{a_{j}\}_{j=1,2,\ldots}$ converges to zero, then for every $\epsilon>0$ there exist a constant $N(\epsilon)$ such that $|a_{j}|<\epsilon$ for all $j>N(\epsilon)$.}. Furthermore, there exists $N_{2}$  such that for all $j$ greater than $N_{2}$ the estimated derivative function along the mode estimates becomes less than $c$, i.e.,  $\hat{f}^{'}_{h,k}(y_{j})<c,\forall j>N_{2}$. Let $N=\max\{N_{1}(\epsilon/2),N_{2}\}$. Then, we have
\begin{align}
\forall  j>N: y_{j} \not \in \bigcup_{i=1}^n I[x_{i},\epsilon], \; y_{j}-\epsilon/2<y_{j+1}<y_{j}+\epsilon/2.
\end{align}
 Let $j>N$ and, without loss of generality, assume $y_{j+1}\ge y_{j}$. We show that $y_{j+2}\ge y_{j+1}$, and hence for $j>N$ the mode estimate sequence will be a non-decreasing sequence. We define sets $D_{1}$, $D_{2}$, and $D_{3}$ as follows
 \begin{align*}
 D_{1}=\{x_{i}: y_{j}>x_{i}\}, \; D_{2}=\{x_{i}: y_{j+1}>x_{i}>y_{j}\},\; D_{3}=\{x_{i}: x_{i}>y_{j+1}\}.
 \end{align*}
 Since $g$ is a strictly decreasing function, then the following inequality holds
 \begin{align}\label{inequality for one set}
 \sum_{x_{i}\in D_{3}}(x_{i}-y_{j+1})g\big(|x_{i}-y_{j}|^2\big)\le \sum_{x_{i}\in D_{3}}(x_{i}-y_{j+1})g\big(|x_{i}-y_{j+1}|^2\big).
 \end{align}
 Using (\ref{modeestimate}), we obtain
 \begin{align} \label{replacing}
  \sum_{x_{i}\in D_{3}}(x_{i}-y_{j+1})g\big(|x_{i}-y_{j}|^2\big)= \sum_{x_{i}\in D_{1}\cup D_{2}}(y_{j+1}-x_{i})g\big(|x_{i}-y_{j}|^2\big).
  \end{align}
  Replacing the left side of (\ref{inequality for one set})  with the right side of (\ref{replacing}), we get
  \begin{align} \label{replacing2}
  \sum_{x_{i}\in D_{1}\cup D_{2}}(y_{j+1}-x_{i})g\big(|x_{i}-y_{j}|^2\big) \le\sum_{x_{i}\in D_{3}}(x_{i}-y_{j+1})g\big(|x_{i}-y_{j+1}|^2\big).
  \end{align}
Adding  $\sum_{x_{i} \in D_{1}\cup D_{2}}(x_{i}-y_{j+1})g(|x_{i}-y_{j+1}|^2)$  to both sides of equation (\ref{replacing2}), gives
  \begin{align}\label{last1}
&\sum_{x_{i} \in D_{1}\cup D_{2}}(y_{j+1}-x_{i})g(|x_{i}-y_{j}|^2)+\sum_{x_{i} \in D_{1}\cup D_{2}}(x_{i}-y_{j+1})g(|x_{i}-y_{j+1}|^2)  \\
&\le \sum_{x_{i} \in D_{3}}(x_{i}-y_{j+1})g(|x_{i}-y_{j+1}|^2)+\sum_{x \in D_{1}\cup D_{2}}(x_{i}-y_{j+1})g(|x_{i}-y_{j+1}|^2). \nonumber
\end{align}
From the properties given in (\ref{consequtive difference}) and (\ref{deri}), we observe that $D_{2}$ is an empty set. Therefore, the left side of the above inequality can be simplified to
\begin{align*}
&\sum_{x_{i} \in D_{1}\cup D_{2}}(y_{j+1}-x_{i})g(|x_{i}-y_{j}|^2)+\sum_{x_{i} \in D_{1}\cup D_{2}}(x_{i}-y_{j+1})(|x_{i}-y_{j+1}|^2) \\
&=\sum_{x_{i} \in D_{1}}(y_{j+1}-x_{i})\Big(g(|x_{i}-y_{j}|^2)-g(|x_{i}-y_{j+1}|^2)\Big) \ge 0.
\end{align*}
Hence, the right side of (\ref{last1}) is nonnegative and we have
\begin{align*}
0\le \sum_{x_{i} \in D_{3}\cup D_{2} \cup D_{1}}(x_{i}-y_{j+1})g(|x_{i}-y_{j+1}|^2).
\end{align*}
This is equivalent to $y_{j+2}\ge y_{j+1}$. Therefore, for all $j>N$ if $y_{j+1}\ge y_{j}$, then $y_{j+2}>y_{j+1}$. By induction for all $j>N$ the sequence $\{y_{j}\}$ will be  monotonically increasing and hence  convergent.\\
For the case that $y_{j+1}\le y_{j}$, we define sets $D_{1}$, $D_{2}$, and $D_{3}$ as follows
\begin{align*}
D_{1}=\{x_{i}|x_{i}<y_{j+1}\}, \; D_{2}=\{x_{i}|y_{j+1}<x_{i}<y_{j}\}, \; D_{3}=\{x_{i}|x_{i}>y_{j}\}.
\end{align*}
Then similar to the previous case, it is straightforward to show that $y_{j+2}\le y_{j+1}$. Therefore, the mode estimate sequence $\{y_{j}\}$ for all $j>N$ becomes a monotonically decreasing and convergent sequence.\\

It remains to prove the monotonicity of the mode estimate sequence for the case that for some $x_{i} \in X$, $\hat{f}_{h,k}^{'}(x_{i})=0$.
Let $\hat{f}^{'}_{h,k}(x_{i}^*)=0$ for some $x_{i}^*\in X$.
If there exists $N$, such that for all $j>N$, there is not any $x_{i}^*$ between $y_{j}$ and $y_{j+1}$, then the previous results can be applied to show that the mode estimate sequence is a monotone sequence. Otherwise, we assume that such $N$ does not exist.
We need the following lemma
\begin{lem}\label{lema}
Consider a fixed point iteration defined by $y_{j+1}=m(y_{j})$, where $m$ is a differentiable function. Let $x^*$ denote a solution of the fixed point problem, i.e.,  $x^*=m(x^*)$ and let $e_{j}$ denote the distance between the fixed point $x^*$ and $y_{j}$, i.e., $e_{j}=|x^*-y_{j}|$, respectively. Then there exists $ \delta$ such that $e_{j+1}=e_{j}|m^{'}(\delta)|$ and $y_{j}<\delta<x^*$ if $y_{j}<x^*$ and  $x^*<\delta<y_{j}$ if $x^*<y_{j}$.
\end{lem}
\begin{proof}
Using the mean value theorem, there exists $ \delta$ such that $y_{j}<\delta<x^*$ (without loss of generality assume $y_{j}<x^*$) and $m(x^*)-m(y_{j})=(x^*-y_{j})m'(\delta)$. Then, we have
\begin{align*}
e_{j+1}=|x^*-y_{j+1}|&=|m(x^*)-m(y_{j})|\\&=|(x^*-y_{j})m'(\delta)|\\
&=|(x^*-y_{j})| |m'(\delta)|\\
&=e_{j}|m'(\delta)|.
\end{align*}
That shows $e_{j+1}=e_{j}|m'(\delta)|$.
\end{proof}

Using lemma $1$, there are three possibilities for $m'(x^*)$ that we check separately:
\begin{enumerate}
\item If  $|m'(x^*)|<1$, then there exists an interval $I=[x^*-\epsilon,x^*+\epsilon]$  such that for all $x \in I$, $|m'(x)|<1$. Hence, if the sequence $\{y_{j}\}$ falls in $I$, then it converges to $x^*$ using lemma (\ref{lema})(since $e_{j}$ becomes a decreasing sequence and finally converges to zero). If the sequence $\{y_{j}\}$ never falls in this interval, then there exists $N$ large enough such that for all $j>N$, $x^*$ is not between $y_{j}$ and $y_{j+1}$, which contradicts the assumption we have made about non-existence of such $N$.
\item If $|m'(x^*)| > 1$, then there is a closed interval $I=[x^*-\epsilon,x^*+\epsilon]$ such that  for all $ x\in I$, we have $|m'(x)|> 1$. For some $j$, let the sequence $y_j$  fall in $I$. Otherwise, we can find large enough $N$ such that for all $j>N$ there is no $x_{i}^*$ between $y_{j}$ and $y_{j+1}$, which contradicts our assumption for non-existence of such $N$. We choose $j$ large enough such that $|y_{j+1}-y_{j}|<\epsilon/2$. There are four possibilities as follows
\begin{enumerate}
\item $x^*-\epsilon \le y_{j} < x^*-\frac{\epsilon}{2},$
\item $x^*-\frac{\epsilon}{2} \le y_{j} < x^*,$
\item $x^* \le y_{j} < x^*+\frac{\epsilon}{2},$
\item $x^*+\frac{\epsilon}{2} \le y_{j} < x^*+\epsilon.$
\end{enumerate}
Let $x^*-\epsilon \le y_{j} < x^*-\frac{\epsilon}{2}$. It is clear that in this case $e_{k+1}> e_{k}$, since  for all $x \in I$,  $m'(x)> 1$. It means that the Euclidean distance between $y_{j+1}$ and $x^*$ is greater than the Euclidean distance between $y_{j}$ and $x^*$ ($y_{j+1}$ is also on the left side of the $x^*$ because it is assumed that $|y_{j+1}-y_{j}|<\epsilon/2$). Therefore, in this case the sequence $y_{j}$ can never fall in the interval $I'=[x^*-\frac{\epsilon}{2},x^*+\frac{\epsilon}{2}]$. Hence, for all $j>N$, there is no $x_{i}^*$ between $y_{j}$ and $y_{j+1}$, which contradicts our assumption about the non-existence of such $N$.(Case~$4$ can be treated exactly in a same).

Let $x^*-\frac{\epsilon}{2} \le y_{j} < x^*$. Also  for all $x \in I$,  $m'(x)> 1$. It is obvious that the Euclidean distance between $y_{j+1}$ and $x^*$ is greater than the Euclidean distance between $y_{j}$ and $x^*$ ($y_{j+1}$ can be  in the left or right side of the $x^*$). In this case, after some finite iterations (Let us say $M$ iterations), the cases $1$ or $4$ will happen and then it can be concluded for all $j>N+M$, the sequence $y_{j} \not \in I'=[x^*-\frac{\epsilon}{2},x^*+\frac{\epsilon}{2}]$, which contradicts our assumption about non-existence of such $N$.
The third case can be treated similar to the second case.

\item If $|m'(x^*)|=1$, then there are three possibilities as follows:
 \begin{enumerate}
 \item $\exists I$ around $x^*$ such that $\forall x \in I$, $m'(x)>1$. This case was discussed before.
 \item $\exists I$ around $x^*$ such that $\forall x \in I$, $m'(x)<1$. This case was discussed before.
 \item $\exists I$ around $x^*$ such that $\forall x \in I$ and $x<x^*$, $m'(x)<1$. Also,  $\forall x \in I$ and $x>x^*$, $m'(x)>1$. In this case, the mode estimate sequence either converges to $x^*$ or there is a closed interval $I'$ around $x^*$ such that $y_{j}$ never falls in that interval. Convergence of the later case is guaranteed according to the above discussion.
\end{enumerate}
\end{enumerate}

This completes the convergence proof of the sequence in the one dimension.

\end{proof}

\textbf{Remarks}
\begin{enumerate}[(a)]
\item The authors in \cite{meer} proved that if a kernel $K$ has a convex, differentiable, and monotonically decreasing profile $k$, then the estimated pdf using the kernel $K$ and the bandwidth $h$ along the mean shift sequence is monotonically increasing and convergent. In other words, they proved the monotonicity and convergence of $\{\hat{f}_{h,K}(y_{j})\}_{j=1,2,\ldots}$, that $\{y_{j}\}_{j=1,2,\ldots}$ is the sequence generated by the mean shift algorithm. It is obvious that the convergence of $\{\hat{f}_{h,K}(y_{j})\}$ does not imply the convergence of the mode estimate sequence $\{y_{j}\}$. The authors assumed that the sequence $\{y_{j}\}$ generated by the mean shift sequence  is a Cauchy sequence, which is not true in general. Hence the proof given in \cite{meer} for the convergence of the mean shift sequence is not correct \cite{ourpaper}.
\item The authors in \cite{finite} assumed that the number of the modes of an estimated pdf is finite and, based on this assumption, they showed that the mean shift sequence $\{y_{j}\}$ converges, but they could not justify their assumption.
    Showing the finiteness of the number of the modes of an estimated pdf is still an open problem and there is not any useful condition to guarantee the finiteness of the number of stationary points of an estimated pdf.
\item Carreira-Perpi\~{n}\'an showed that the mean shift algorithm with the Gaussian kernel is an EM algorithm \cite{per} and therefore the generated sequence $\{y_{j}\}$ converges to the modes of the estimated pdf. A counterexample for the convergence of the EM algorithm is given in \cite{emnot}, which shows in general the EM algorithm may not converge.
\item So far, only the convergence of $\hat{f}_{h,K}(y_{j})$ is proved in the literature. Theorem $1$ provides sufficient conditions to guarantee convergence of the mean shift sequence in a one-dimensional space. It assures that under certain conditions the mode estimate sequence generated by the mean shift algorithm is a monotone and convergent sequence in the one-dimensional space. The convergence of the mean shift sequence for higher dimensions (when dimensionality of input data is greater than one) has not yet been proved.

\end{enumerate}
\section{Simulation Results}\label{section4}
We carried out a series of simulations to demonstrate the results of Theorem~$1$. We assumed that the input data are generated by one of two normal distributions with the mean values +$3$ and -$3$ and a variance of $1$, i.e., $x\sim \mathcal{N}(3,1)$ or $x\sim \mathcal{N}(-3,1)$.
The total number of the observed data is $1000$, such that $500$ samples are generated by the first normal distribution and the rest of the samples are generated by the second normal distribution. For the mean shift algorithm we used the Gaussian kernel that satisfies conditions given in Theorem~$1$. The bandwidth $h$ is fixed to $1$, and we stop the mean shift iterations if the distance between two consecutive mode estimates becomes less that $0.0005$. Figure~$1$ shows the convergence of the mean shift algorithm for ten different initializations. In each case, the mean shift algorithm generates a monotone sequence converging to one of the two available modes. For example, in the top left of Figure~$1$, the mean shift algorithm starts from the point $6.045$ and iteratively update the estimated mode. As is expected, the algorithm generates a decreasing sequence that converges to $3$. In the top right of Figure~$1$, the mean shift algorithm initializes to $-6.575$ and it generates an increasing sequence that converges to $-3$. The rest of the graphes in Figure~$1$ demonstrate the convergence of the mean shift algorithm to either -$3$ or $3$. In each case, based on the initial value, the algorithm generates an increasing or decreasing sequence in order to estimate a mode of the pdf. Table~$1$ shows the values of the mode estimate sequence as a function of the number of the iterations. The starting points are $6.045$, $-6.575$, $0.905$, $-0.575$, $4.457$, $-4.759$, $0.588$, $-0.602$, $5.076$, and $-5.160$. It can be observed from Table~$1$ that as the number of iterations increase, the mode estimate sequence converges to either $-3$ or $3$.  We will get similar results when a pdf estimate has multiple modes (more than two modes) and the mean shift algorithm generates a monotone sequence that converges to a mode that is closer to its initial value.


\begin{figure}[htp]
\begin{center}
\begin{tabular}{cc}
{\includegraphics[width=0.85\textwidth]{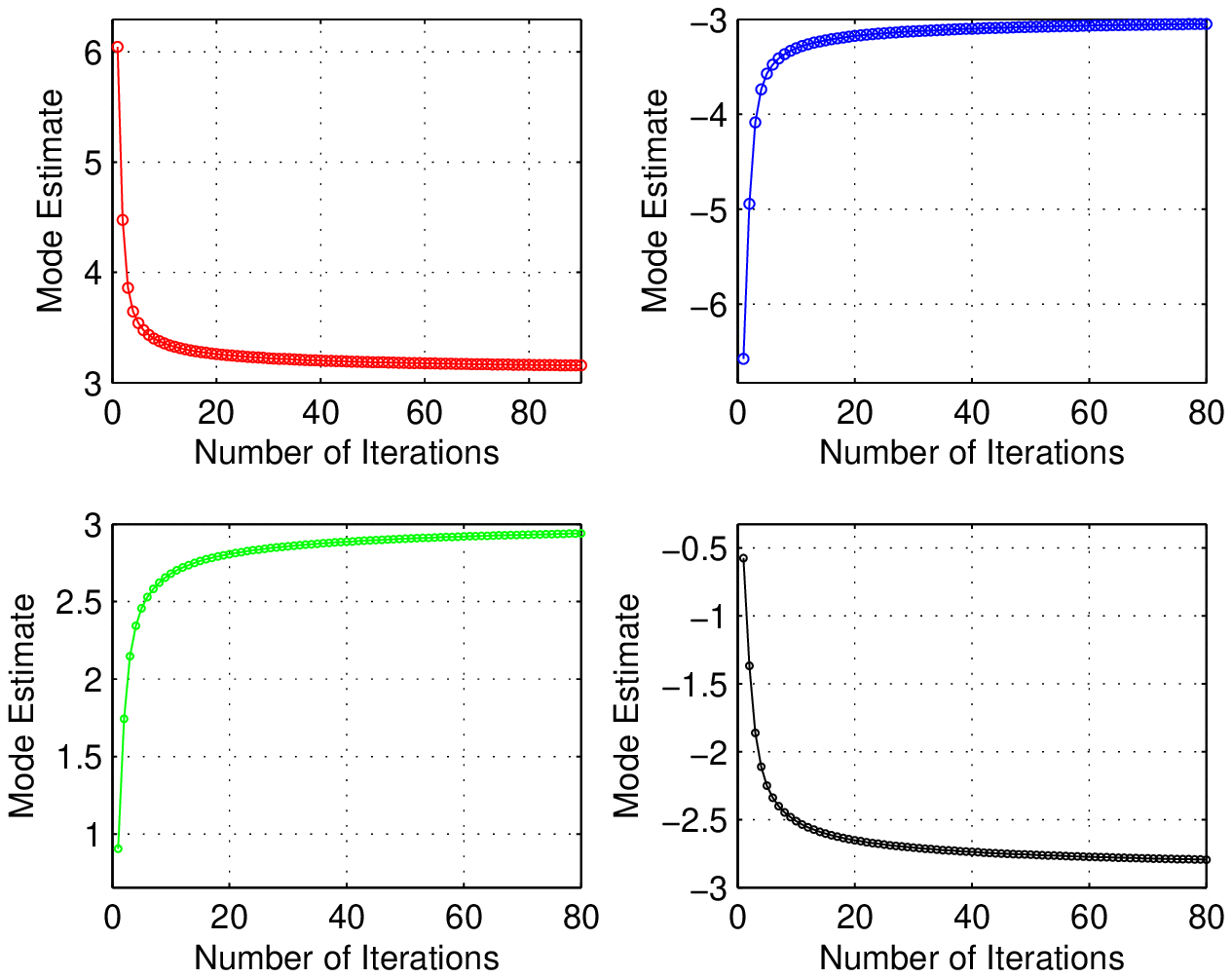}}
\\
{\includegraphics[width=0.85\textwidth]{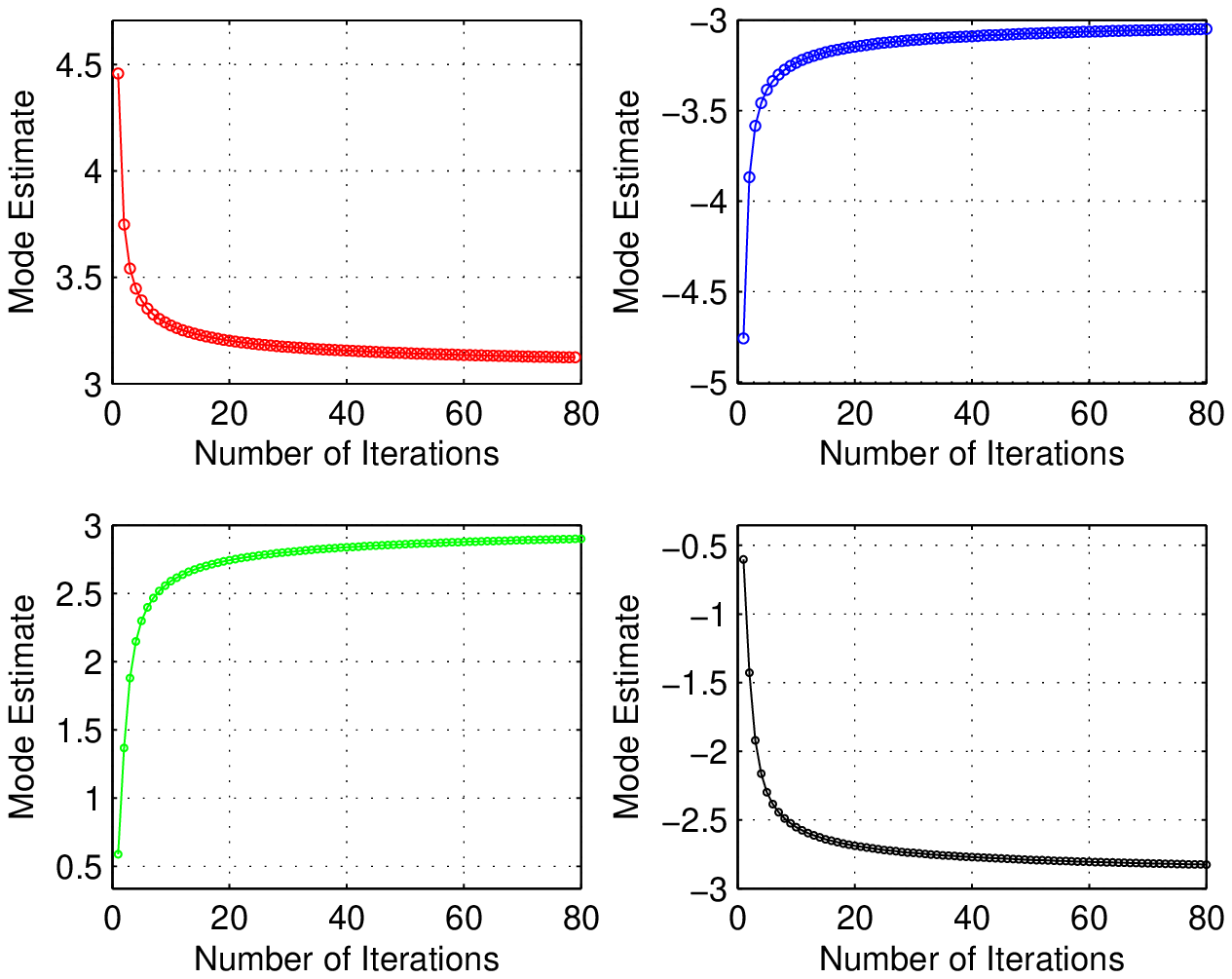}}
\\
{\includegraphics[width=0.85\textwidth]{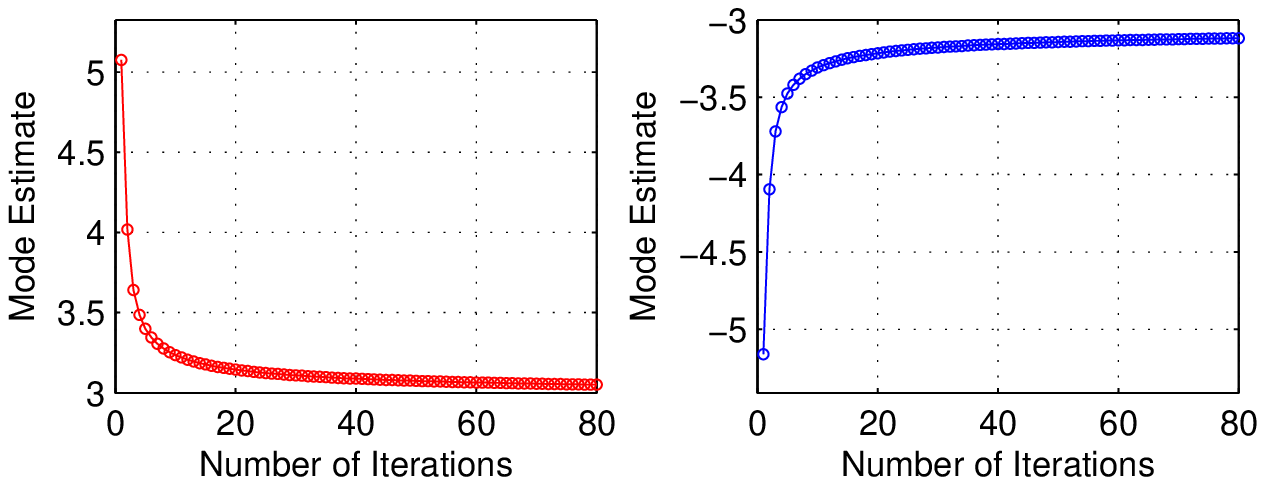}}
\\
\end{tabular}
\caption[ Simulation using the Mean shift algorithm ]
{\small{The mode estimate sequence generated by the mean shift algorithm with different initial values. The $x$-axis represnets the number of iterations and $y$-axis represents the value of the estimated mode. The bandwidth $h$ is equal to one and the iterations stop when the difference between two consecutive mode estimates becomes less than $0.0005$. In the above simulations, the mean shift algorithm was initialized to the following values (from top to bottom, and from left to right): $6.045$, $-6.575$, $0.905$, $-0.575$, $4.457$, $-4.759$, $0.588$, $-0.602$, $5.076$, and $-5.160$. It can be observed from the simulations that based on the initial value, the mean shift algorithm generates a convergent monotone sequence.\label{figone}}}
\end{center}
\end{figure}


\begin{table}

\label{table:circle quantization}
\begin{center}
    \begin{tabular}{| l | l | l | l | l | l | l | l |}
    \hline
    Number of iteration & 1 & 5 & 10 & 20 & 40 & 80 & 81 \\ \hline
    Estimated mode & 6.045 & 3.540 & 3.356 & 3.260 & 3.201 & 3.091 & 3.091 \\ \hline
    Estimated mode  & -6.575 &  -3.572 & -3.304 & -3.175  & -3.098 & -3.047 & -3.047  \\ \hline
    Estimated mode  & 0.905 &  2.456 & 2.680 & 2.806  & 2.887 & 2.940 & 2.941  \\ \hline
    Estimated mode  & -0.575 &  -2.250 & -2.512 & -2.652  & -2.738 & -2.795 & -2.797  \\ \hline
    Estimated mode  & 4.457 &  3.391 & 3.273 & 3.202  & 3.155 & 3.124 & 3.123  \\ \hline
    Estimated mode  & -4.759 &  -3.384 & -3.234 & -3.145  & -3.087 & -3.048 & -3.047  \\ \hline
    Estimated mode  & 0.588 &  2.299 & 2.589 & 2.743  & 2.837 & 2.900 & 2.901  \\ \hline
    Estimated mode  & -0.602 &  -2.298 & -2.553 & -2.688  & -2.771 & -2.855 & -2.856  \\ \hline
    Estimated mode  & 5.076 &  3.400 & 3.236 & 3.145  & 3.088 & 3.051 & 3.050  \\ \hline
    Estimated mode  & -5.160 &  -3.477 & -3.308 & -3.215  & -3.157 & -3.118 & -3.117  \\ \hline

    \end{tabular}
\end{center}
\caption[Mean square distortion for the noisy circle ]{\small{The mode estimate sequence generated by the mean shift algorithm when it starts from ten different points. The bandwidth $h$ is fixed to $1$, and the number of iterations goes from $1$ to $81$. The algorithm stops when the difference between consecutive mode estimates becomes negligible.}}
\end{table}
\section{Conclusion}\label{section5}
The mean shift algorithm is a  simple non-parametric iterative technique for finding modes of an estimated pdf. Although the mean shift algorithm has been used in many pattern recognition and machine vision applications, its convergence has not yet been proved. In this paper we proved the convergence of the mean shift algorithm in the one-dimensional space. Specifically, we proved that if the kernel $K$ has a convex, differentiable, and strictly decreasing profile $k$, then the mode estimate sequence $\{y_{j}\}_{j=1,2,\ldots}$ generated by the mean shift algorithm in the one-dimensional space is a monotone and bounded sequence and therefore the sequence converges. The convergence of the mean shift algorithm in $n$dimensional ($n>1$) space is the subject of future studies.
\bibliography{strings,refs}

\end{document}